\documentclass[sigconf]{acmart}

\usepackage{graphicx}
\usepackage{amsfonts}
\usepackage{todonotes}
\usepackage{amsthm}

\newtheorem{theorem}{Theorem}
\theoremstyle{definition}
\newtheorem{definition}{Definition}

\usepackage[noend]{algpseudocode}
\usepackage[ruled]{algorithm}
\usepackage{multirow}
\usepackage{subfig}
\usepackage{stmaryrd}
\usepackage{url}

\usepackage{tikz}
\usetikzlibrary{arrows.meta, shapes.geometric, positioning}

\newcommand{\R}{\mathbb{R}}
\newcommand{\feats}{\mathcal{X}}
\newcommand{\dataset}{\mathcal{D}}
\newcommand{\dtrain}{\dataset_{train}}
\newcommand{\dtrigger}{\dataset_{trigger}}
\newcommand{\dtest}{\dataset_{test}}
\newcommand{\hyps}{\mathcal{H}}
\newcommand{\labels}{\mathcal{Y}}
\newcommand{\convert}[1]{\llbracket #1 \rrbracket}

\newcommand{\so}[1]{{#1}}

\AtBeginDocument{%
  }

\begin{document}

\title{Watermarking Decision Tree Ensembles}

\author{Stefano Calzavara}
\email{stefano.calzavara@unive.it}
\affiliation{%
  \institution{Università Ca' Foscari Venezia}
  \city{Venice}
  \country{Italy}
}

\author{Lorenzo Cazzaro}
\email{lorenzo.cazzaro@unive.it}
\affiliation{%
  \institution{Università Ca' Foscari Venezia}
  \city{Venice}
  \country{Italy}
}

\author{Donald Gera}
\email{892604@stud.unive.it}
\affiliation{%
  \institution{Università Ca' Foscari Venezia}
  \city{Venice}
  \country{Italy}
}

\author{Salvatore Orlando}
\email{orlando@unive.it }
\affiliation{%
  \institution{Università Ca' Foscari Venezia}
  \city{Venice}
  \country{Italy}
}

\begin{abstract}
Protecting the intellectual property of machine learning models is a hot topic and many watermarking schemes for deep neural networks have been proposed in the literature. Unfortunately, prior work largely neglected the investigation of watermarking techniques for other types of models, including decision tree ensembles, which are a state-of-the-art model for classification tasks on non-perceptual data. In this paper, we present the first watermarking scheme designed for decision tree ensembles, focusing in particular on random forest models. We discuss watermark creation and verification, presenting a thorough security analysis with respect to possible attacks. We finally perform an experimental evaluation of the proposed scheme, showing excellent results in terms of accuracy and security against the most relevant threats.
\end{abstract}



\received{20 February 2007}
\received[revised]{12 March 2009}
\received[accepted]{5 June 2009}

\maketitle

\section{Introduction}

Machine learning models are pervasively used and are often considered intellectual property of the legitimate parties who have trained them. This is often a consequence of the incredible number of computational resources required for model training. For example, it has been estimated that even a relatively small model like GPT-3 might cost around 5M dollars for training on the cloud~\cite{GPT3}. This motivated a significant amount of research on \emph{model watermarking}, in particular for deep neural networks~\cite{Boenisch21, LiWB21, UchidaNSS17}. A watermark is a piece of identifying information which is embedded into the model to claim copyright, without affecting model accuracy too much. Different watermarking schemes have been proposed with different properties, e.g., the literature distinguishes between zero-bit watermarking~\cite{NambaS19, TangJDWJH23, AdiBCPK18, KuribayashiTF20, SzyllerAMA21, RouhaniCK19, ZhangCLFZZCY20} and multi-bit watermarking~\cite{abs-1904-00344, UchidaNSS17, TartaglioneGCB20, FanNC19, GuoP18, TangDH23, RouhaniCK19}, based on the amount of information embedded in the watermark.

Although watermarking received a great deal of attention in the field of deep neural networks, it was not carefully investigated for other types of machine learning models for different reasons. First, some models are shallow in the sense that they are not over-parameterized and redundant, lacking room to effectively embed watermarks. Moreover, traditional machine learning models require way less computational resources for training than deep neural networks. Yet, the process of collecting high-quality training data, cleaning them and in some cases even manual labeling them should be performed for any type of supervised learning algorithm. This process is generally time-consuming and expensive, thus making copyright protection of highly effective models trained over high-quality datasets an urgent practical need.

\paragraph{Contributions} In this paper, we present the first watermarking scheme designed for \emph{decision tree ensembles}, which are a state-of-the-art model for classification tasks on non-perceptual data. We discuss watermark creation and verification, presenting a thorough security analysis with respect to possible attacks, including watermark detection, watermark suppression and watermark forgery. In particular, we prove that the watermark forgery problem is NP-hard, thus providing a theoretical guarantee about the effectiveness of the proposed method. We finally perform an experimental evaluation of the proposed scheme, showing excellent results in terms of accuracy and security against the most relevant threats.
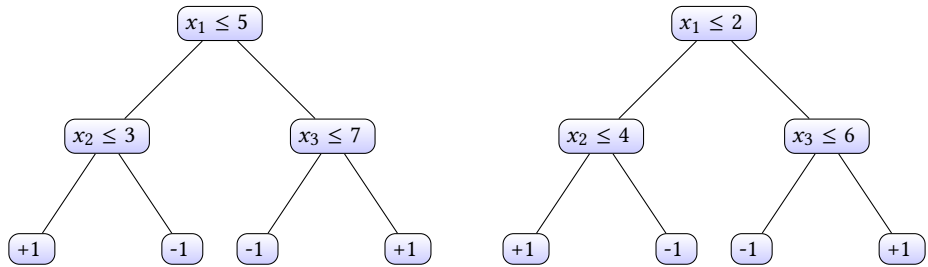
\begin{figure*}[t]
\centering
\begin{tikzpicture}[level 1/.style={sibling distance=3cm},level 2/.style={sibling distance=2cm},  every node/.style = {shape=rectangle, rounded corners, draw, align=center, top color=white, bottom color=blue!20}]

Tree 1
\node { $x_1 \leq 5$ }
 child { node { $x_2 \leq 3$ }
   child { node { +1 } }
   child { node { -1 } }
 }
 child { node { $x_3 \leq 7$ }
   child { node { -1 } }
   child { node { +1 } }
 };

Tree 2
\node[right=6cm] { $x_1 \leq 2$ }
 child { node { $x_2 \leq 4$ }
   child { node { +1 } }
   child { node { -1 } }
 }
 child { node { $x_3 \leq 6$ }
   child { node { -1 } }
   child { node { +1 } }
 };

\end{tikzpicture}
\caption{Example of a decision tree ensemble with two trees.}
\label{fig:ensemble}
\end{figure*}

\section{Background}
Let $\feats \subseteq \R^d$ be a $d$-dimensional vector space of real-valued \textit{features}. An \emph{instance} $\vec{x} \in \feats$ is a $d$-dimensional feature vector $\langle x_1, x_2, \ldots, x_d \rangle$ representing an object in the vector space $\feats$. Each instance is assigned a class label $y \in \labels$ by an unknown function $f: \feats \rightarrow \labels$. Supervised learning algorithms learn a \emph{classifier} $g: \feats \rightarrow \labels$ from a \emph{training set} of correctly labeled instances $\dtrain = \{(\vec{x}_i,f(\vec{x}_i))\}_i$, with the goal of approximating the target function $f$ as accurately as possible. The performance of classifiers is assessed on a \emph{test set} of correctly labeled instances $\dtest = \{(\vec{z}_i,f(\vec{z}_i))\}_i$, disjoint from the training set, yet drawn from the same data distribution.

In this paper, we focus on a specific class of supervised learning algorithms training traditional \emph{binary decision trees} for classification~\cite{BreimanFOS84}. Decision trees can be inductively defined as follows: a decision tree $t$ is either a leaf $L(y)$ for some label $y \in \labels$ or an internal node $N(f \leq v, t_l, t_r)$, where $f \in \{1,\ldots,d\}$ identifies a feature, $v \in \R$ is a threshold for the feature, and $t_l,t_r$ are decision trees (left and right child). Decision trees are learned by initially putting all the training set into the root of the tree and by recursively splitting leaves (initially: the root) by identifying the threshold therein leading to the best split of the training data, e.g., the one with the highest information gain, thus transforming the split leaf into a new internal node. At test time, the instance $\vec{x}$ traverses the tree $t$ until it reaches a leaf $L(y)$, which returns the prediction $y$, denoted by $t(\vec{x}) = y$. Specifically, for each traversed tree node $N(f \leq v, t_l, t_r)$, $\vec{x}$ falls into the left sub-tree $t_l$ if $x_f \leq v$, and into the right sub-tree $t_r$ otherwise. To improve their performance, decision trees are often combined into an \emph{ensemble} $T = \langle t_1,\ldots,t_m \rangle$, which aggregates individual tree predictions, e.g., by performing majority voting. Figure~\ref{fig:ensemble} shows an example ensemble including $m = 2$ decision trees.
\section{Ensemble Watermarking}
We first motivate our key design choices and we introduce the threat model considered in this work. We then present our watermarking scheme and security analysis.

\subsection{Design Choices and Threat Model}
We explain the key design choices and the threat model using the terminology of a recent survey~\cite{Boenisch21}. Our watermark is embedded during the training phase by means of a \emph{trigger set}, i.e., a set of instances evoking unusual prediction behavior in the watermarked model. The watermark is \emph{multi-bit}, i.e., it embeds a binary signature of the model owner into the model behavior, and provides \emph{authentication}, i.e., the legitimate model owner may claim copyright in front of a legal entity. Verification is \emph{black-box}, i.e., the legitimate model owner may access the potentially stolen model solely through queries and has no visibility of the model parameters.

We assume that the attacker has illegitimate white-box access to the watermarked model. We also assume that the attacker does not modify the model in any way, e.g., due to some form of integrity protection or because they do not want to risk reducing model accuracy at test time. For example, the model might be integrated as part of a production software that allows the attacker to inspect and query the model, but not modify it. This is in line with prior work which observed that it is extremely difficult to draw a line between adapting an existing model and creating an entirely different model on its own~\cite{GuoP18}. Our watermarking scheme is designed to mitigate the following threats:
\begin{enumerate}
    \item \emph{Watermark detection}: the attacker should be unable to detect the presence of the watermark. This is important to limit the attacker's knowledge, making it easier to catch them red-handed when they use the model and preventing room for additional attacks against the watermark.

    \item \emph{Watermark suppression}: the attacker should be unable to identify the queries involving the trigger set, otherwise they might change the model predictions over the trigger set to make black-box verification fail, thus rendering the watermark useless in practice.

    \item \emph{Watermark forgery}: the attacker should be unable to construct a valid watermark, otherwise they may unduly claim ownership of the stolen model.
\end{enumerate}

\subsection{Watermarking Scheme}
Our method is reminiscent of the watermarking scheme proposed for deep neural networks by Guo and Potkonjak~\cite{GuoP18}. Their scheme generates a binary signature of the model owner and embeds it into the training data in order to generate the trigger set. Instances from the trigger set obtain different labels than the original data points that they were based on, hence the model exhibits abnormal behavior on them. Watermark verification is performed by confirming the abnormal behavior of the model on the trigger set, in their case a significant accuracy drop with respect to a traditional model trained over the same training data. In our case, we instead use the signature to encode a specific model behavior that the trees in the ensembles are required to show on the trigger set.

We focus on random forest models without bootstrap, leaving the generalization to more sophisticated ensemble methods to future work. In these models, each tree is a classifier trained on a subset of the features of the entire training set and the final prediction is computed by aggregating individual tree predictions, e.g., using majority voting. We assume that the output of the ensemble is the sequence of the class predictions performed by each tree. For example, in R the \texttt{predict.all} field is exactly used for this purpose and a similar behavior may be easily encoded in sklearn by creating a wrapper of the \texttt{RandomForestClassifier} class. For simplicity, we focus on binary classification tasks, i.e., the set of labels $\labels$ contains just a positive class +1 and a negative class -1. Multi-class classification can be supported by encoding it in terms of multiple binary classification tasks.

\begin{algorithm}[t]
\caption{Watermark creation algorithm}
\label{alg:watermark}
\begin{algorithmic}[1]
\Function{TrainWithTrigger}{$\dtrain, \dtrigger, m, \hyps$}
    \State{$\hyps \gets \Call{Adjust}{\hyps}$} \Comment{Adjust $\hyps$ to hide the watermark}
    \State{$W \gets \{(\vec{x},y) \mapsto 1 ~|~ (\vec{x},y) \in \dtrain\}$} \Comment{Sample weights for training}
    \State{$T \gets \Call{TrainRandomForest}{\dtrain, m, \hyps, W}$}
    \While{$\exists t_i \in T: \exists (\vec{x},y) \in \dtrigger: t_i(\vec{x}) \neq y$}
        \For{$(\vec{x}, y) \in \dtrigger$}
            \State{$W[(\vec{x},y)] \gets W[(\vec{x},y)] + 1$} \Comment{Increase sample weights}
        \EndFor
        \State{$T \gets \Call{TrainRandomForest}{\dtrain, m, \hyps, W}$}
    \EndWhile
    \State{\Return{$T$}}
\EndFunction
\State{}
\Function{Watermark}{$\dtrain, m, \sigma, k$}
\State{$\hyps \gets \Call{GridSearch}{\dtrain,m}$} \Comment{Find hyper-parameters}
\State{$\dtrigger \gets \Call{Sample}{\dtrain, k}$} \Comment{Random sampling}
\State{$m' \gets |\{1 \leq i \leq m ~|~ \sigma_i = 0\}|$}
\State{$T_0 \gets \Call{TrainWithTrigger}{\dtrain, \dtrigger, m', \hyps}$}
\State{$\dtrigger' \gets \{(\vec{x},-y) ~|~ (\vec{x},y) \in \dtrigger \}$} \Comment{Flip labels}
\State{$\dtrain \gets (\dtrain \setminus \dtrigger) \cup \dtrigger'$}
\State{$T_1 \gets \Call{TrainWithTrigger}{\dtrain, \dtrigger', m - m', \hyps}$}
\State{$T \gets \{\}$}
\For{$i \in \{1, \ldots, m\}$}
    \If{$\sigma_i = 0$}
       $T[i] \gets \Call{GetNextTree}{T_0}$
    \Else\
        $T[i] \gets \Call{GetNextTree}{T_1}$
    \EndIf
\EndFor
\State{\Return{$\langle T, \dtrigger \rangle$}}
\EndFunction
\end{algorithmic}
\end{algorithm}

Our watermarking scheme is shown in the \textsc{Watermark} function of Algorithm~\ref{alg:watermark} (lines 11-23). It takes as input a training set $\dtrain$, the number of trees in the ensemble $m$, the signature of the model owner $\sigma$ (of length $m$) and the size of the trigger set $k \ll |\dtrain|$. We denote by $m'$ the number of bits of $\sigma$ set to 0, hence $m-m'$ is the number of bits of $\sigma$ set to 1. Associated with $\sigma$, we have a subset of samples of the training set, denoted by $\dtrigger$, where each tree of the ensemble must either classify correctly or misclassify based on the setting of the bits of $\sigma$. Specifically, the $i$-th tree in the ensemble is forced to classify correctly if the $i$-th bit of $\sigma$ is set to 0, and to misclassify otherwise. This specific output pattern is used for watermark verification and, as we argue, it is difficult to reproduce out of the trigger set, thus mitigating the risk of watermark forgery.

The algorithm first uses grid search to find the best model hyper-parameters $\hyps$ for an ensemble of $m$ trees. After sampling a trigger set $\dtrigger \subseteq \dtrain$ of size $k$, the algorithm trains two ensembles $T_0$ and $T_1$ with hyper-parameters $\hyps$ using the \textsc{TrainWithTrigger} function (lines 1-9). The function uses sample weighting to force a specific model behavior on $\dtrigger$. Specifically, all the $m'$ trees of $T_0$ perform correct predictions for all samples of $\dtrigger$, while all the $m-m'$ trees of $T_1$ misclassify by predicting the opposite label. Before training $T_0$ and $T_1$, the function adjusts $\hyps$ to make the two ensembles look structurally more similar to each other to prevent watermark detection. In particular, we observe that trees in $T_1$ may have a stronger tendency at overfitting than trees in $T_0$. The reason is that $T_1$ operates abnormally on the trigger set, i.e., we force prediction errors there, which often pushes trees in $T_1$ to grow larger than trees in $T_0$. To mitigate this effect, we train a standard tree ensemble with the hyper-parameters $\hyps$ and we adjust them as follows. First, we identify the average and standard deviation of the different hyper-parameters (depth and number of leaves) observed in the trained model. We then update $\hyps$ to the difference between the average and the standard deviation, i.e., we force both depth and number of leaves to be lower than the average. This simple heuristic prevents $T_1$'s trees from growing much more that the ones in $T_0$, while still overfitting the expected wrong output on the trigger set. Moreover, $T_0$ does not deviate too much with respect to a standard ensemble trained with the goal of minimizing prediction errors over the training data. The effect is that the trees in $T_0$ and $T_1$ look similar to each other, while largely preserving model accuracy.

At the end of the algorithm, the watermarked ensemble $T$ is constructed by picking its $i$-th tree from $T_0$ if the $i$-th bit of $\sigma$ is 0 and from $T_1$ otherwise. The algorithm returns a pair $\langle T,\dtrigger \rangle$ including the watermarked ensemble and the trigger set. As for watermark verification, assume Alice has watermarked her model using our algorithm and wants to sue Bob as an illegitimate user of the model. Alice gives to the legal authority Charlie the following information: her signature $\sigma$, the trigger set $\dtrigger$ and a set of test data $\dtest$ such that $\dtrigger \subseteq \dtest$. Charlie feeds $\dtest$ to Bob's model and retrieves the predictions associated with $\dtrigger$. Charlie then verifies that all the instances in $\dtrigger$ are classified correctly by some $t_i \in T$ iff $\sigma_i = 0$. The use of $\dtest$ is useful to prevent watermark suppression by disguising $\dtrigger$ among other instances fed in the verification phase.

\begin{figure*}[t]
\centering
\begin{tikzpicture}[level 1/.style={sibling distance=3cm},level 2/.style={sibling distance=2cm},  every node/.style = {shape=rectangle, rounded corners, draw, align=center, top color=white, bottom color=blue!20}]

\node { $x_1 \leq 0$ }
  child { node { $x_2 \leq 0$ }
    child { node { -1 } }
    child { node { +1 } }
  }
  child { node { +1 }
  };

\node[right=5cm] { $x_2 \leq 0$ }
  child { node { $x_3 \leq 0$ }
    child { node { $x_4 \leq 0$ }
        child { node { +1 } }
        child { node { - 1} } }
    child { node { +1 } }
  }
  child { node { +1 }
  };
\end{tikzpicture}
\caption{Conversion of the example formula $(x_1 \vee x_2) \wedge (x_2 \vee x_3 \vee \neg x_4)$ into a tree ensemble.}
\label{fig:reduction}
\end{figure*}
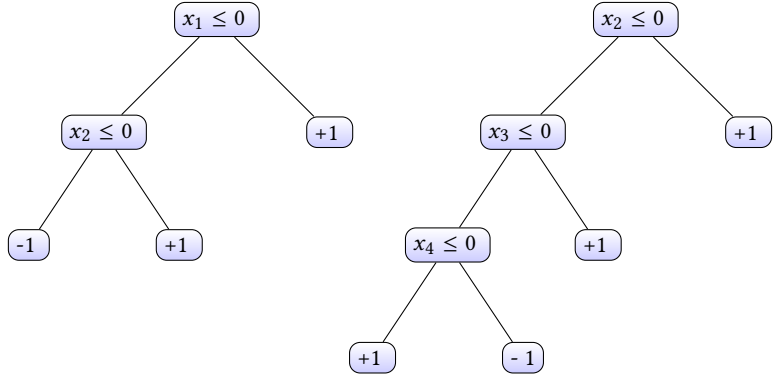

\subsection{Security Analysis}
We argue about the security of our watermarking scheme and we present an empirical validation of our claims in Section~\ref{sec:experiments}. Our scheme is robust against watermark detection because the trees of $T$ are trained using hyper-parameters tuned by training traditional tree ensembles, i.e., the watermarked ensemble has a similar structure to a standard model. Although hyper-parameters are adjusted as explained before, the attacker does not know the optimal value of the hyper-parameters and cannot infer the adoption of watermarking from the ensemble structure alone. Most importantly, trees in $T_0$ and $T_1$ are trained using adjusted hyper-parameters forcing them to look similar in terms of depth and number of leaves, hence the correct signature $\sigma$ cannot be reconstructed by inspecting the structure of the trees in the ensemble.

Moreover, our scheme is robust against watermark suppression because $\dtrigger$ is a subset of $\dtrain$. This means that $\dtrigger$ is sampled from the same distribution of the training data, which are themselves assumed to be representative of the distribution of the test data (otherwise, learning would be ineffective). In other words, data in the trigger set are indistinguishable from standard test data and cannot be easily detected by the attacker during the watermark verification phase, which means that the attacker cannot maliciously adapt the model output on $\dtrigger$.

Finally, our scheme is robust against watermark forgery. Assume that the attacker does not know the signature $\sigma$ and the trigger set $\dtrigger$, but generates a fake signature $\sigma'$ and tries to forge a trigger set $\dtrigger'$ where the watermarked model exhibits the output pattern required by $\sigma'$. This is equivalent to solving a satisfiability problem for a logical formula encoding the expected model output. To exemplify, consider the tiny ensemble with two trees in Figure~\ref{fig:ensemble} and let $\sigma' = 01$ be the fake signature. Forging a positive instance $\langle \vec{x},+1 \rangle$ matching the fake signature $\sigma'$ is equivalent to finding a satisfying assignment for the following logical formula:
\begin{align*}
\phi & \triangleq ((x_1 \leq 5 \wedge x_2 \leq 3) \vee (x_1 > 5 \wedge x_3 > 7)) \\
& \wedge ((x_1 \leq 2 \wedge x_2 > 4) \vee (x_1 > 2 \wedge x_3 \leq 6)).
\end{align*}
A similar reasoning may be applied to forge a negative instance $\langle \vec{x},-1 \rangle$. In this toy example, it is easy to see that $\vec{x} = \langle x_1, x_2, x_3 \rangle = \langle 4,3,5 \rangle$ is a possible satisfying assignment for $\phi$. However, as the size of the ensemble grows larger, such formulas become increasingly more difficult to solve and might not even admit any satisfying assignment. Note that formulas like $\phi$ do not define a system of linear inequalities, because they involve the disjunction operator and require solving an instance of the \textit{Boolean satisfiability problem} (SAT), which is NP-hard in general. Indeed, we can provide a formal NP-hardness proof for the watermark forgery problem.
\begin{definition}
The \emph{watermark forgery problem} is defined as follows: given a tree ensemble $T$, a label $y \in \{-1,+1\}$ and a signature $\sigma$, find an instance $\vec{x}$ such that $\forall t_i \in T: t_i(\vec{x}) = y \Leftrightarrow \sigma_i = 0$.
\end{definition}

\begin{theorem}
The watermark forgery problem is NP-hard.
\end{theorem}
\begin{proof}
We show a reduction from 3SAT to watermark forgery, i.e., we show that if there exists a polynomial time algorithm to solve the watermark forgery problem, then there exists a polynomial time algorithm to solve 3SAT, which is known to be NP-complete. This proves that there is no polynomial time algorithm to solve the watermark forgery problem. First of all, we recap the 3SAT problem. A boolean variable $x$ is a variable that can only take value true or false, while a literal $l$ is a boolean variable or its negation. A 3CNF formula $\phi$ is a formula of the form $\psi_1 \wedge \ldots \wedge \psi_k$ with $k \geq 1$, where each $\psi_i$ is a disjunction of three or less literals. More formally, 3CNF formulas $\phi$ are generated by the following context-free grammar: $l ::= x ~|~ \neg x \quad \psi ::= l ~|~ l \vee l ~|~ l \vee l \vee l \quad \phi ::= \psi ~|~ \phi \wedge \phi$.

An example of a 3CNF formula is $(x_1 \vee x_2) \wedge (x_2 \vee x_3 \vee \neg x_4)$.

The 3SAT problem requires, given a 3CNF formula $\phi$, to find the values of the boolean variables that make the formula true (or return a message that no such values exist). The reduction operates by first constructing an ensemble $T$ including a decision tree $t_i$ of depth three or less for each sub-formula $\psi_i$ in $\phi$, using prediction paths to encode the truth value of the literals therein. In particular, each internal node of the tree branches over the value of a variable $x_j$ occurring in $\psi_i$ with threshold 0, using the left child to represent the value false and the right child to represent the value true. We set just one of the children to have label +1, based on whether setting $x_j$ to false or to true is a sufficient condition for the satisfiability of the sub-formula $\psi_i$. The conversion from 3CNF formulas to ensembles is intuitive and exemplified in Figure~\ref{fig:reduction} for the example formula given above.

Generalization to arbitrary 3CNF formulas is conceptually simple, but technical to define. In particular, we define a conversion function $\convert{\cdot}$ by induction on the structure of the formulas as follows:
\[
\convert{l} =
\begin{cases}
N(x \leq 0, L(-1), L(+1)) & \textnormal{if } l = x \\
N(x \leq 0, L(+1), L(-1)) & \textnormal{if } l = \neg x
\end{cases}
\]
\[
\convert{\psi} =
\begin{cases}
\convert{l} & \textnormal{if } \psi = l \\
N(x \leq 0, \convert{\psi'}, L(+1)) & \textnormal{if } \psi = x \vee \psi' \\
N(x \leq 0, L(+1), \convert{\psi'}) & \textnormal{if } \psi = \neg x \vee \psi' \\
\end{cases}
\]
\[
\convert{\phi} =
\begin{cases}
\convert{\psi} & \textnormal{if } \phi = \psi \\
\langle \convert{\phi_1}, \convert{\phi_2} \rangle & \textnormal{if } \phi = \phi_1 \wedge \phi_2.
\end{cases}
\]

By construction, we have that $\phi$ is satisfiable if and only if the watermark forgery problem has a solution for the ensemble $\convert{\phi}$ using label $y = +1$ and signature $\sigma = \langle 0, \ldots 0 \rangle$. Indeed, the leaves of a tree $t_i$ with label $+1$ identify prediction paths encoding sufficient conditions for the satisfiability of the sub-formula $\psi_i$, hence finding a positive instance $\vec{x}$ such that $t_i(\vec{x}) = +1$ is equivalent to finding a satisfying assignment for $\psi_i$. The bits of $\sigma$ are all set to 0 because $\phi$ is satisfiable if and only if all the sub-formulas $\psi_i$ are satisfiable, being $\phi$ a conjunction. If a solution $\vec{x}$ is found for the watermark forgery problem, we can translate into a value assignment for 3SAT by having each variable $x_j$ set to true if and only if the $j$-th component of the solution is positive.
\end{proof}
\section{Experimental Evaluation}
\label{sec:experiments}
We implemented the proposed watermarking scheme on top of the \texttt{sklearn} library and we make our code publicly available to support reproducibility.\footnote{\url{https://zenodo.org/doi/10.5281/zenodo.13269530}} We here evaluate the accuracy of watermarked models and the security of our watermarking scheme on public datasets (MNIST2-6, breast-cancer and ijcnn1) normalized in the interval $[0,1]$. Note that MNIST2-6 includes digits representing numbers 2 and 6 from the traditional MNIST dataset, while ijcnn1 has been reduced to 10,000 instances using stratified random sampling to speed up the experimental evaluation. Table~\ref{tab:datasets} reports the most relevant dataset statistics, showing that the considered datasets are diverse in terms of number of instances, number of features and class distribution.

\begin{table}[t]
    \centering
    \caption{Dataset statistics.}
    \begin{tabular}{c|c|c|c}
    \toprule
    \textbf{Dataset} & \textbf{Instances} & \textbf{Features} & \textbf{Distribution} \\
    \midrule
    MNIST2-6 & 13,866 & 784 & $51\% / 49\%$ \\
    breast-cancer & 569 & 30 & $63\% / 37\%$ \\
    ijcnn1 & 20,000 & 22 & $10\% / 90\%$ \\
    \bottomrule
    \end{tabular}
\label{tab:datasets}
\end{table}

\subsection{Accuracy Evaluation}
Since watermarked models force a specific prediction pattern over the trigger set, their predictive power on the test data might be penalized. In our first set of experiments, we evaluate the accuracy loss introduced by our watermarking scheme. Figure~\ref{fig:accuracy1} plots how accuracy downgrades for increasing sizes of the trigger set, given a fixed randomly generated signature including 50\% of the bits set to 1. The figure shows that the accuracy loss is limited in general and even negligible when the size of the trigger set does not exceed 2\%.

\begin{figure}[t]
  \centering
    \subfloat{\label{fig:accuracy1}\includegraphics[width=.4\textwidth]{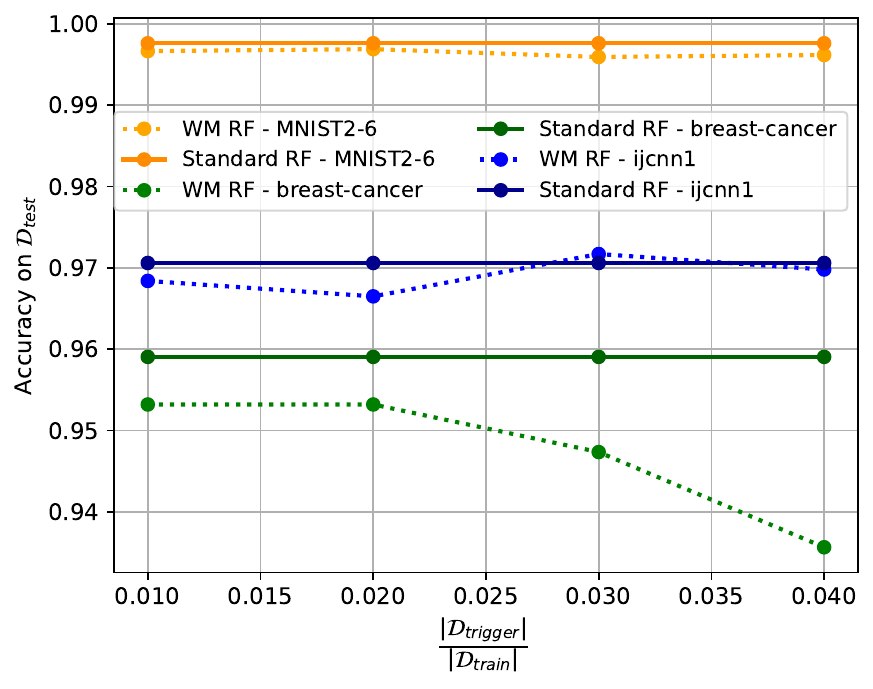}}
  \hfill
    \subfloat{\label{fig:accuracy2}\includegraphics[width=.4\textwidth]{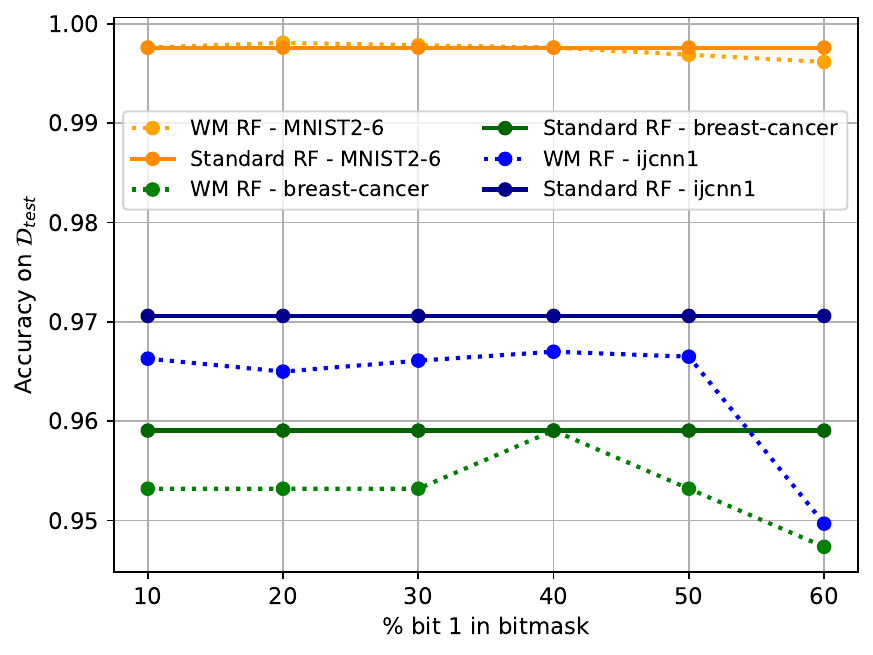}}
  \caption{Accuracy of watermarked models on the test set when varying the percentage of training instances included in $\dtrigger$ (top figure) and the percentage of bits set to 1 in the signature $\sigma$ (bottom figure).}
\end{figure}

Of course, the number of bits set to 1 in the signature might also impact the accuracy of the watermarked model, because such bits denote forced prediction errors. Figure~\ref{fig:accuracy2} shows how accuracy changes when we increase the number of bits set to 1 in the signature, given a fixed trigger set (including 2\% of the training data). Again, the accuracy loss is small in practice, with the largest drop in accuracy amounting to around two points.

\subsection{Security Evaluation}
We focus in particular on watermark detection and watermark forgery, because protection against watermark suppression is immediately achieved by construction. We assume that $\sigma$ includes 50\% of the bits set to 1 and $\dtrigger$ includes 2\% of the training set. 

\begin{table*}[t]
    \centering
    \caption{Number of trees correctly/wrongly associated with their bits, using two watermark detection strategies. For each dataset, \textit{mean} and \textit{standard deviation} of ``Depth'' and ``\#leaves'' are reported in round brackets.}
    \begin{tabular}{c|c|c|c|c}
    \toprule
    \textbf{Dataset} & \textbf{Hyper-Parameters} &   \textbf{\#correct} & \textbf{\#wrong} & \textbf{\#uncertain} \\
    \midrule
    \multirow{2}{*}{MNIST2-6} & Depth (19.82 - 2.69) & \textcolor{red}{31} / \textcolor{blue}{57} & \textcolor{red}{11} / \textcolor{blue}{33} & \textcolor{red}{48} / \textcolor{blue}{0} \\
    & \#leaves (229.99 - 0.10) & \textcolor{red}{1} / \textcolor{blue}{46} & \textcolor{red}{0} / \textcolor{blue}{44} & \textcolor{red}{89} / \textcolor{blue}{0} \\
    \midrule
    \multirow{2}{*}{breast-cancer} & Depth (7.03 - 0.81) & \textcolor{red}{34} /
    \textcolor{blue}{46} & \textcolor{red}{9} / \textcolor{blue}{24} & \textcolor{red}{27} / \textcolor{blue}{0} \\
    & \#leaves (18.90 - 0.45) & \textcolor{red}{4} / \textcolor{blue}{39} & \textcolor{red}{0} / \textcolor{blue}{31} & \textcolor{red}{66} / \textcolor{blue}{0} \\
    \midrule
    \multirow{2}{*}{ijcnn1} & Depth (18.00 - 0.00) & \textcolor{red}{0} / \textcolor{blue}{40} & \textcolor{red}{0} / \textcolor{blue}{40} & \textcolor{red}{80} / \textcolor{blue}{0} \\
    & \#leaves (498.88 - 5.86) & \textcolor{red}{0} /
    \textcolor{blue}{37} & \textcolor{red}{3} / \textcolor{blue}{43} & \textcolor{red}{77} / \textcolor{blue}{0} \\
    \bottomrule
    \end{tabular}
\label{tab:detection-mean-std}
\end{table*}

\subsubsection{Watermark Detection}
We compare the depth and the number of leaves of the trees corresponding to bits set to 0 and to 1 in the signature $\sigma$ to understand whether there are relevant differences leaking information about $\sigma$. This is a significant threat, because trees associated with a bit set to 1 are forced to make prediction errors in the trigger set, hence they might grow larger than the other trees when trying to achieve overfitting. We simulate two watermark detection strategies by means of the following experiment: given a hyper-parameter like depth or number of leaves, the attacker computes its mean and standard deviation over the ensemble. Intuitively, ``small'' trees are more likely to be associated with bit 0 and ``large'' trees are more likely to be associated with bit 1. To formalize this intuition, in our first strategy the attacker associates bit 0 with all trees falling below the difference of the mean and standard deviation, and bit 1 to all trees falling above the sum of mean and standard deviation; all the other trees \so{around the mean} correspond to \textit{uncertain} cases, where the attacker might try random guessing. \so{Note that this strategy may produce a large number of uncertain cases, thus making random guessing of them infeasible for the attacker. However, the technique is interesting because we can check whether it can correctly identify at least the rest of the trees. The second strategy does not produce uncertain trees, as it uses the mean as a \so{sharp} threshold to determine whether a tree is associated with bit 0 or 1.} 
Table~\ref{tab:detection-mean-std} reports the results, showing that both the attack strategies are ineffective. 
The first strategy (in \textcolor{red}{red}) yields a huge number of uncertain cases\so{, but surprisingly it also produces wrong predictions for the rest of the trees.} 
The second strategy (in \textcolor{blue}{blue}) has no uncertainty, but produces many prediction errors and is unable to reconstruct the signature. \so{Finally, we can observe that standard deviation values are relatively small compared to the values of the associated means. Therefore, the trees trained by our techniques are all similar to each other, thus making it very difficult for an attacker to identify $\sigma$.}

\subsubsection{Watermark Forgery}
To show security against watermark forgery, we simulate a scenario where the attacker generates a fake signature $\sigma'$ and tries to forge a trigger set $\dtrigger'$ where the watermarked model exhibits the output required by $\sigma'$. We showed that this requires solving an NP-hard problem, however recent advances in automated verification enable dealing with large inputs even for computationally intensive problems, hence we complement our theoretical analysis with empirical evidence. We implement our forgery attempts by generating 10 random signatures and solving a satisfiability problem for a logical formula encoding the expected model output using Z3, a state-of-the-art SMT solver~\cite{MouraB08}. For each fake signature, we iterate over all the instances in the test set and we look for a satisfying assignment for our logical formula, while requiring that the $L_\infty$-distance between the solution and the original test instance is bounded by some $0 < \varepsilon < 1$. The distance constraint is useful to ensure that the forged trigger set $\dtrigger'$ is reminiscent of real test instances. We are in fact assuming that, as usual, the test set has the same distribution of the training set. 

Our experiment shows different results on the different datasets. In the case of breast-cancer, the forged trigger set reaches at most 14\% of the size of the original trigger set, even when setting a high $\varepsilon = 0.9$. This is explained by the fact that Z3 does not find satisfying assignments for most of the logical formulas, hence the legitimate model owner is the only one who is able to present a trigger set of significant size. In the case of ijcnn1, instead, the forged trigger set is just 1\% of the size of the original trigger set on average for $\varepsilon = 0.1$. Forging a trigger set of the same size as the original trigger set for $\varepsilon > 0.1$ does not scale, already requiring more than four hours for a single bitmask for $\varepsilon = 0.3$. The reason is that the ensemble for ijcnn1 contains more than twice the leaves of the ensembles for the other two datasets, making the satisfiability problem more difficult. The results are more interesting for the MNIST2-6 dataset and we visualize them in Figure~\ref{fig:forgery}. The figure shows that, when $\varepsilon$ increases, it becomes easier to forge trigger sets of comparable size to the original trigger set. However, the amount of distortion required by the forgery makes it easy to detect such malicious attempts, because the size of the forged trigger set become comparable to the original only when $\varepsilon \geq 0.7$. Figure~\ref{fig:forgery-dissimilar} shows three forged images of $\dtrigger'$ for increasing values of $\varepsilon \in \{0.3, 0.5, 0.7\}$. As we can see, the image with the highest amount of distortion is rather blurry and quite far from the original image. Indeed, a standard decision tree ensemble achieves 0.99 accuracy on the original trigger set, while its accuracy drops to 0.62 on the forged trigger set.

\begin{figure}[t]
  \centering
  \includegraphics[width=.35\textwidth]{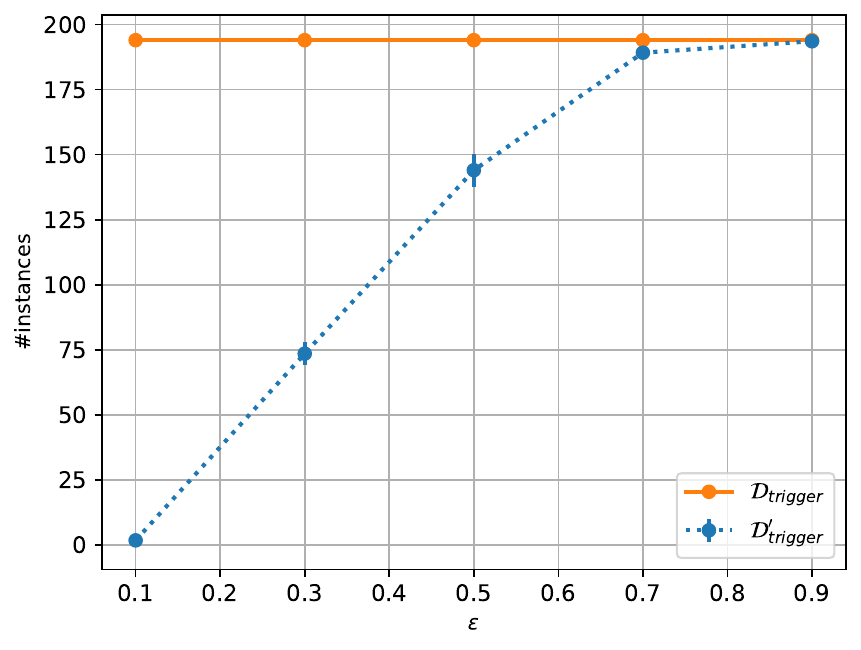}
  \hfill
  \caption{Size of the forged trigger set $\dtrigger'$ when varying the amount of distortion $\varepsilon$ on the MNIST2-6 dataset.}
  \label{fig:forgery}
\end{figure}

\begin{figure}[t]
  \centering
\includegraphics[width=.45\textwidth]{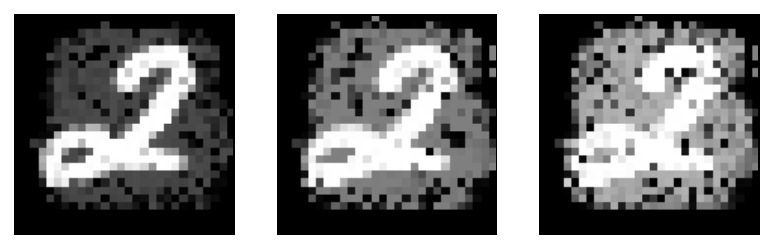}
  \caption{Instances generated by Z3 for $\varepsilon \in \{0.3, 0.5, 0.7\}$.}
  \label{fig:forgery-dissimilar}
\end{figure}

\section{Conclusion}
We proposed the first watermarking scheme designed for decision tree ensembles and we motivated the security of our construction. Our experimental evaluation on public datasets shows promising results, because watermarked models largely preserve their accuracy and are robust against relevant attacks. As future work, we plan to extend our security analysis to more powerful attackers, e.g., who are able to modify the watermarked model and forge trigger sets using more sophisticated strategies. We would also like to generalize our watermarking scheme to more advanced decision tree ensembles, such as those trained using gradient boosting.


\bibliographystyle{ACM-Reference-Format}
\bibliography{biblio.bib}


\end{document}